\documentclass[11pt]{article}
\usepackage{amsfonts,amssymb,amsmath,pdfsync,numbysec,xcolor,bbm}
\hsize \textwidth
\advance \hsize by -\marginparwidth
\evensidemargin \oddsidemargin
\usepackage{amssymb}
\usepackage{amsthm}
\usepackage[shortlabels]{enumitem}
\usepackage[normalem]{ulem}
\usepackage{cancel}
\usepackage{scalerel}
\usepackage{multicol}
\usepackage{xfrac}
\usepackage{authblk}

\newtheorem{assumption}{Assumption}

\usepackage{mathrsfs}%,showlabels}
\usepackage[toc,page]{appendix}

\advance\hoffset by 5mm

\newcommand{\R}{\mathbb R}
\newcommand{\eps}{\varepsilon}
\newcommand{\commentout}[1]{}
\newcommand{\E}{{\mathbb E}}
\renewcommand{\Pr}{{\mathbb P}}

\renewcommand{\phi}{\varphi}

\usepackage[colorlinks=true,pdfpagemode=UseNone,urlcolor=blue,linkcolor=blue,citecolor=blue,breaklinks=true]{hyperref}
\usepackage{cleveref}

\newtheorem{thm}{Theorem}[section]
\newtheorem{lem}[thm]{Lemma}

\newtheorem{rem}[thm]{Remark}

\newcommand{\be}{\begin{equation}}
\newcommand{\ee}{\end{equation}}

\newcommand{\bal}{\begin{aligned}}
\newcommand{\enbal}{\end{aligned}}

\newcommand{\one}{{\mathbbm{1}}}

\newcommand{\bhat}[1]{\expandafter\hat#1} 

\renewcommand{\d}{\partial}

\numberbysection
\numberwithin{equation}{section}

\begin{document}
\title{Convergence of stochastic gradient descent under a local \L{}ojasiewicz condition for deep neural networks}
\author{Jing An and Jianfeng Lu}
\affil{Duke University}
\date{}
\maketitle

\begin{abstract}

 We study the convergence of stochastic gradient descent (SGD) for non-convex objective functions. We establish the local convergence with positive probability under the local \L{}ojasiewicz condition introduced by Chatterjee in \cite{chatterjee2022convergence} and an additional local structural assumption of the loss function landscape. A key component of our proof is to ensure that the whole trajectories of SGD stay inside the local region with a positive probability. We also provide examples of neural networks with finite widths such that our assumptions hold. 

% We extend the  convergence result of Chatterjee \cite{chatterjee2022convergence} by considering the stochastic gradient descent (SGD) for non-convex objective functions. With minimal additional assumptions that can be realized by finitely wide neural networks, we prove that if we initialize inside a local region where the \L{}ojasiewicz condition holds, with a positive probability, the stochastic gradient iterates converge to a global minimum inside this region. A key component of our proof is to ensure that the whole trajectories of SGD stay inside the local region with a positive probability. For that, we assume the SGD noise scales with the objective function, which is called machine learning noise and achievable in many real examples. Furthermore, we provide a negative argument to show why using the boundedness of noise with Robbins-Monro type step sizes is not enough to keep the key component valid.

\end{abstract}

\section{Introduction}
The stochastic gradient descent (SGD) and its variants are widely applied in machine learning problems due to its computational efficiency and generalization performance. A typical empirical loss function for the training writes as
\be 
F(\theta) = \E_{\xi\sim \mathcal{D}}[f(\theta; \xi)],
\ee
where $\xi$ denotes the random sampling from the training data set following the distribution $\mathcal{D}$. A standard SGD iteration to train parameters $\theta\in\R^n$  is of the form 
\be\label{alg:sgd}
\theta_{k+1} = \theta_k - \eta_k\nabla f(\theta_k;\xi_k).
\ee
Here, the step size $\eta_k$ can be either a fixed constant or iteration-adapted, and $\nabla f(\theta_k;\xi_k)$ is a unbiased stochastic estimate of the gradient $\nabla F(\theta_k)$,
 induced by the sampling of the dataset.

%\subsection*{Related works}

%\begin{itemize}
%    \item 
%Review of stochastic optimization...

The convergence of SGD for convex objective functions has been well established, and we give an incomplete list of works \cite{bertsekas2011incremental, bottou2018optimization, gower2019sgd, Kushner2003StochasticAA, moulines2011non, nemirovski2009robust} here for reference. 
%However, it is more realistic to study the SGD for non-convex optimization problems, because 
%training tasks are usually combined with complex neural networks and the stochastic gradient algorithms perform particularly well in non-convex optimization 
 Since SGD algorithms in practice are often applied to non-convex problems in machine learning such as complex neural networks and demonstrate great empirical success, much attention has been drawn to study the SGD in non-convex optimization 
\cite{cutkosky2019momentum, goyal2017accurate, neyshabur2015path}. 
Compared with convex optimization, the behavior of stochastic gradient algorithms over the non-convex landscape is unfortunately much less understood. 
It is natural to investigate whether stochastic gradient algorithms converge through the training, and what minimum they converge to in non-convex problems. However, these questions are noticeably  challenging since  the trajectory of stochastic iterates is more difficult to track due to the noise. Most available results are limited. For example, works such as \cite{allen2018make, allen2018natasha, ghadimi2016accelerated, ward2020adagrad} provide convergence guarantees to a critical point in terms of quantifying the vanishing of $\nabla F$, but little information is given on what critical points that SGD converges to.  Many convergence results are based on global assumptions on the objective function, including the global Poylak-\L{}ojasiewicz condition \cite{lei2017non, liu2022loss}, the global  quasar-convexity \cite{gower2021sgd}, or assumptions of weak convexity and global boundedness of iterates \cite{duchi2018stochastic, stich2019unified}. Those global assumptions are often not realistic, at least they cannot cover general multi-modal landscapes.

%\item 
%PL conditions etc and neural networks...

More specifically for optimization problems for deep neural network architectures, most convergence results are obtained in the overparametrized regime, which means that the number of neurons grow at least polynomially with respect to the sample size. For example, works including \cite{chizat2019lazy, jacot2018neural, soltanolkotabi2018theoretical, zou2020gradient} consider wide neural networks, which essentially linearize the problem by extremely large widths. Particularly in such settings, Poylak-\L{}ojasiewicz type conditions are shown to be satisfied, and they thus prove convergence with linear rates \cite{allen2019convergence, liu2022loss}. Let us also mention convergence results of shallow neural networks in the mean field regime \cite{chizat2018global, mei2018mean, pham2020global}, while the convergence has not been  fully established for deep neural networks.

Convergence results are very limited for neural networks with finite widths and depths, and we refer to \cite{chatterjee2022convergence, jentzen2021existence, liuoptimization} for recent progresses in terms of the convergence of gradient descent in such scenario. 
In particular, \cite{chatterjee2022convergence} constructs  feedforward neural networks with smooth and strictly increasing activation functions, with the input dimension being greater than or equal to the number of data points. Such neural networks satisfy a local version of the \L{}ojasiewicz inequality, and the convergence of gradient descent to a global minimum given appropriate initialization are fully analyzed. In this work,  our goal is to extend the convergence result in \cite{chatterjee2022convergence} to stochastic gradient descent, with minimal additional assumptions added to the loss function $F(\theta)$.
%\end{itemize}

In this work, we extend Chatterjee's convergence result to SGD for non-convex objective functions with minimal additional assumptions applicable to finitely wide neural networks. Our main result Theorem \ref{thm:main} asserts that, with a positive probability, SGD converges to a zero minimum within a locally initialized region satisfying the \L{}ojasiewicz condition (Assumption 1). In particular, our proof relies on assuming that the noise scales with the objective function (Assumption 4), and in the end we provide an negative argument showing that convergence with the bounded noise and Robbins-Monro type step sizes can fail in specific scenarios (Theorem \ref{lem:counterexample}).

\subsection*{Notation} Throughout the note, $|\cdot|$ denotes the Euclidean norm, $B(\theta,r)$ denotes an Euclidean ball with radius $r$ centered at $\theta$. Unless otherwise specified, the expectation $\E= \E_{\xi\sim \mathcal{D}}$, and the gradients $\nabla =\nabla_{\theta}$.

\section{Preliminaries}
We will refer the major assumption used in  \cite{chatterjee2022convergence} as the \emph{local \L{}ojasiewicz condition}. Let us first recall the definition as in \cite{chatterjee2022convergence}: Given the dimension $d$, let $F: \R^d\to [0,\infty)$ be a non-negative objective function. For any $\theta_0\in \R^d$ and $r>0$, we define
\be\label{cond:alpha}
\alpha(\theta_0, r):= \inf_{B(\theta_0, r), F(\theta)\neq 0} \frac{|\nabla F(\theta)|^2}{F(\theta)}.
\ee
If $F(\theta)=0$ for all $\theta\in B(\theta_0,r)$, then we set $\alpha(\theta_0,r)=\infty$. With the same initialziation $\theta_0$ and (\ref{cond:alpha}), the main assumption is
\begin{assumption}[local \L{}ojasiewicz condition] We assume that for some $r>0$,
\be \label{assump:local_PL}
4F(\theta_0)< r^2\alpha(\theta_0, r).
\ee
\end{assumption}
 Detailed discussions on this \emph{local \L{}ojasiewicz condition} can be found in \cite{chatterjee2022convergence}. Let us mention that the local \L{}ojasiewicz condition does have limitations as standard Polyak-\L{}ojasiewicz conditions have. However, aiming at finding the zero minimum,  the choice of our initialization (\ref{assump:local_PL}) has excluded the possibilities of saddles or sub-optimal local minima that exist in neural networks.

\begin{assumption}[$C_L$-smoothness]
Furthermore, we assume that for a compact set $\mathcal{K}$, there exist a constant $ C_L>0$ such that
\be\label{smooth}
 |\nabla F(\theta_1)-\nabla F(\theta_2)|\leq C_L |\theta_1-\theta_2|
\ee
for any $\theta_1, \theta_2\in \mathcal{K}$.
\end{assumption}
Based on that, we can obtain a local growth control of $|\nabla F|$ . The following result is a local version of \cite[Lemma B.1]{wojtowytsch2021stochastic}. The proof is very similar except that some modification is needed for the localization with the local boundedness of $|\nabla F|$. We provide the proof in the Appendix.
\begin{lem}\label{auxiliary0}
Suppose that $F:\R^d\to \R$ is a non-negative function. Assume $\nabla F$ is Lipschitz continuous in a compact set $\mathcal{K}$ with the constant $C_L$, and there exists $\bar C>0$ such that $\max_{\theta\in\mathcal{K}}|\nabla F(\theta)| = \bar C$. Then there exists a compact set $\tilde{\mathcal{K}}\subset \mathcal{K}$, where $\text{dist}(\d\mathcal{K}, \d\mathcal{\tilde K})$ depends on $C_L$ and $\bar C$, such that for any $\theta\in\tilde{\mathcal{K}}$,
\be\label{growth}
|\nabla F(\theta)|^2\leq 2C_L F(\theta).
\ee
\end{lem}
Throughout the note, we assume that $B(\theta_0, r)\subset\mathcal{\tilde K}$ for a given radius $r>0$.

\subsection*{Additional assumptions} 
Besides the local \L{}ojasiewicz condition, we need an additional structural assumption on $F$.  Given a large radius $R$ with which (\ref{assump:local_PL}) is satisfied, it is natural to assume that the  $F(\theta)$ is bounded away from zero near the domain boundary.
\begin{assumption}[Confinement near boundary]
For some constant $M_0>0$,
\be\label{assump:levelset}
F(\theta)\geq M_0,\quad \text{for all}~\theta\in B(\theta_0, R)\setminus B(\theta_0,R-1).
\ee
% \item For some constant $M_1>0$,
% \be\label{assump:diff}
% |F(y)-F(x)|\geq M_1, \quad \text{for all}~ x\in B(\theta_0, R-1) ~\text{and all}~y\in B(\theta_0,R)^c.
% \ee
\end{assumption} 
Here we set the annulus radius to be $1$  for simplicity. The Assumption (\ref{assump:levelset})  can include more general $F$ by applying geometric transformations to $F$, or linear transformations to the dataset to change the annulus radius.

The assumption (\ref{assump:levelset}) is in order to ensure that the support of the zero minimum $F(\theta^*)=0$  is away from the boundary $\partial B(\theta_0, R)$. In \cite[Theorem 2.1]{chatterjee2022convergence}, Chatterjee constructed a feedforward neural network with finite width and depth that satisfies the local \L{}ojasiewicz condition (\ref{assump:local_PL}). We will verify in the Appendix that such a subgroup of feedforward neural networks can satisfy (\ref{assump:levelset}) as well. 

\subsection*{Stochastic gradients}
To analyze the convergence of stochastic gradient algorithms, it is unavoidable to assume some structure of the gradient noise. We may write the stochastic gradient in the decomposition form
\be\label{jan271}
\nabla f(\theta; \xi) = \nabla F(\theta) +Z(\theta; \xi),
\ee
where the noise term $Z(\theta; \xi)$ is unbiased 
\be 
\E[Z(\theta; \xi)] = 0.
\ee
In most papers, usually two kinds of structural assumptions are imposed on the noise. We take a brief review. 
\begin{itemize}
    \item 
If one plans to analyze SGD with non-constant step sizes $\eta_k$, i.e., in the Robbins–Monro flavor \cite{robbins1951stochastic}
\be\label{assump:rm}
\sum_{k=1}^{\infty} \eta_k = \infty,\quad \sum_{k=1}^{\infty} \eta_k^{1+m/2}<\infty,~\text{with}~m\geq 2,
\ee
then it is typical to assume the bounded moments of stochastic gradients, that is, for some constant $c>0$ and $m\geq 2$,
\be\label{assump:bdd}
\E[|Z(\theta; \xi)|^m]\leq c^m<\infty.
\ee
The bounded variance ($m=2$) assumption is standard in stochastic optimization, and we refer to classical books, lecture notes and papers \cite{benaim2006dynamics, nesterov2003introductory, polyak1987introduction, karimi2016linear} on this setup. The noise boundedness and adaptive step sizes have been a popular combination to establish convergence, for example, \cite{ mertikopoulos2020almost} provides a SGD convergence based on the local convexity assumption, \cite{fehrman2020convergence} gives the convergence of SGD to the local manifold of minima by estimating $\mathbb{P}(F(\theta_k)-\inf_{\theta\in\R^d}F(\theta)\geq\epsilon)$. We also mention works such as \cite{reddi2016stochastic} that introduces the stochastic variance reduction method, which is motivated by improving the convergence rate for $\min_{0\leq k\leq N-1}\E[|\nabla F(\theta_k)|^2]$. 

\item On the other hand, if one just wants to establish global convergence for a fixed step size $\eta$, additional assumptions on how the stochastic gradient being related to the loss function can help. One option, according to \cite{wojtowytsch2021stochastic}, is that
\begin{assumption}
 For some constant $\sigma>0$,
\be\label{assump:jan271}
\nabla f(\theta; \xi) = \nabla F(\theta) + \sqrt{\sigma F(\theta)} Z_{\theta, \xi},
\ee
with 
\be\label{assump:jan272}
\E[Z_{\theta, \xi}] = 0,\quad \E[|Z_{\theta, \xi}|^2] \leq 1.
\ee
\end{assumption}
Here $\sqrt{\sigma F(\theta_k)} Z_{\theta, \xi}$ is named as the machine learning noise, which scales with the function value. We refer to \cite[Section 2.5]{wojtowytsch2021stochastic} for many types of machine learning problems satisfying the assumption (\ref{assump:jan271})-(\ref{assump:jan272}). A more relaxed and general version, considered in \cite{degradient}, is that there exists a monotonically increasing function $\varrho: \R_+\to \R_+$ so that
\be\label{assump:feb801}
\Pr\big(|\nabla f(\theta; \xi)-\nabla F(\theta)|^2\geq t \varrho(F(\theta))\big)\leq e^{-t}.
\ee

% Here, we assume the condition (\ref{smooth}) holds, and thus (\ref{growth}) is true. With that, if we consider the problems such as $F$ is in the form of empirical sums
% \be 
% F(\theta) = \frac{1}{n}\sum_{i=1}^n f_i(\theta),
% \ee
% then we can control the variance of stochastic gradients
% \be 
% \E[|Z(\theta,\xi)|^2]\leq 2\E[|\nabla f(\theta,\xi)|^2]+2|\nabla F(\theta)|^2
% \ee
\end{itemize}
In this note, we will present the convergence of SGD under the second type assumption (\ref{assump:jan271})-(\ref{assump:jan272}). Moreover, we will discuss why convergence under assumptions (\ref{assump:rm})-(\ref{assump:bdd}) can fail for the local \L{}ojasiewicz condition.

\subsection*{Comparison to previous convergence results}
Our motivation is in the same vein as  \cite{chatterjee2022convergence}: we assume the local \L{}ojasiewicz condition for the initialization so that the convergence result may apply to feedforward neural networks of bounded widths and depths.

Several convergence results for SGD in non-convex optimization have been developed in recent years, and we list a few here and highlight their differences:
\begin{itemize}
    \item \cite{mertikopoulos2020almost}: Convergence to a local minima $\theta^*$ that are Hurwicz regular, i.e., $\nabla^2 F(\theta^*)\succ 0$. The key difference is that they assume that there exists $a>0$ such that $\langle \nabla F(\theta), \theta-\theta^*\rangle \geq a |\theta-\theta^*|^2$ for all $\theta$ in a convex compact neighborhood of $\theta^*$.
    \item \cite{wojtowytsch2021stochastic}: It assumes that the \L{}ojasiewicz condition holds in an $\epsilon$-sublevel set  of $F$. As a consequence, a convergence rate for $F(\theta)\to 0$ is obtained, but the proof cannot track where the global minimizer  $x^*$ locates.
\end{itemize}
Similar to the result in \cite{chatterjee2022convergence} for gradient descents, we will show that under our Assumptions 1-4 as in the setup, with some computable probability, stochastic iterates will convergence to the zero minimum almost surely. We need to choose the initialization in a ball where (\ref{assump:local_PL}) holds, and as a consequence, the miminizer also lies in the same ball without assuming its existence a priori, similar to \cite{chatterjee2022convergence}. Compared to \cite{wojtowytsch2021stochastic}, our results is Euclidean as we can track the stochastic trajectories in the training.

\section{Main result}
Given the assumptions we describe in the previous section, we present the convergence of SGD with a quantitative rate and a step-size bound. We point out that (\ref{assump:jan271})-(\ref{assump:jan272}) assuming noise in SGD being scaled with the objective function play an important role here.
\begin{thm}\label{thm:main}
We choose an initialization $\theta_0$ with a radius $R>1$ such that (\ref{assump:local_PL}) holds. Let $F$ be a loss function such that its gradient is Lipschitz continuous (\ref{smooth}), and (\ref{assump:levelset}) holds for $B(\theta_0, R)$, Then for every $\delta>0$, there exists $\eps>0$ such that if $F(\theta_0)\leq \eps$, with probability at least $1-\delta$, we have $\theta_k\in B(\theta_0, R-1)$ for all $k\in\mathbb{N}$.

Conditioned on the event that $\theta_k\in B(\theta_0, R-1)$ for all $k\in\mathbb{N}$, we have 
\be 
\lim_{k\to \infty} \beta^k F(\theta_k) =0
\ee
almost surely for every $\beta\in[1,\rho^{-1})$, where
\be 
\rho =1-\eta \alpha + \frac{\eta^2}{2}C_L(2C_L+\sigma)\in(0, 1),
\ee
if the step size $\eta$ satisfies $0<\eta \leq \min\{ \frac{1}{\alpha}, \frac{\alpha}{4C_L(2C_L+\sigma)}\}$. Moreover, conditioned on the same event, $\theta_k$ converges almost surely to a point $\theta_*$ where $F(\theta_*)=0$.
\end{thm}
\begin{proof}
Let us  define the event where all the iterates up to $k$-th step stay in a ball $B(\theta_0, r)$:
\begin{align}\label{event}
    E_k(r) := \bigcap_{j=0}^k\{|\theta_j-\theta_0|\leq r\},
\end{align}
for some $r\in(0, R]$ to be determined later.
We denote $\mathcal{F}_k$ as the filtration generated by $ \xi_1,\cdots, \xi_{k-1}$. 

Our proof strategy is roughly outlined as follows:
\begin{enumerate}
    \item [(1)] Conditioned on the event $E_k(R-1)$, it can be shown that the expectation of loss $F(\theta_k)$ exhibits a contraction by a factor $\rho\in (0,1)$.
    \item [(2)] Given that, the expectation of the travel distance $(\theta_{k+1}-\theta_k) \one_{E_k(R-1)}$ can be bounded by $C\rho^k/2$ with some constant $C>0$. Therefore, the total distance $\sum_{k=0}^{\infty}(\theta_{k+1}-\theta_k) \one_{E_k(R-1)}$ is finite with high probability. 
    \item [(3)] What remains is to show the event $E_{\infty}(R-1)$ contained in all $E_k(R-1)$ will happen with a positive probability. This can be done by the estimates of (\ref{jan311}) and (\ref{jan312}). Conditioned on $E_{\infty}(R-1)$, $\{\theta_k\}$ form a Cauchy sequence and $F(\theta_k)$ converges to the zero minimum.
\end{enumerate}

Take the interpolation $\theta_{k+s} = \theta_k - s\eta \nabla f(\theta_k,\xi_k)$ for $s\in[0,1]$, we have that
\begin{equation}\label{s_0}
 \begin{aligned}
 F(\theta_{k+1})&-F(\theta_k) = \int_0^1\frac{d}{ds} F\bigl(\theta_k-s\eta \nabla f(\theta_k; \xi_k)\bigr) ds\\
 &=-\eta \int_0^1 \nabla F(\theta_{k+s})\cdot \nabla f(\theta_k;\xi_k) ds\\
 &= -\eta \int_0^1 \nabla F(\theta_{k})\cdot \nabla f(\theta_k;\xi_k) ds-\eta \int_0^1  (\nabla F(\theta_{k+s})-\nabla F(\theta_k))\cdot \nabla f(\theta_k;\xi_k) ds\\
 &\leq -\eta \nabla F(\theta_k) \cdot \nabla f(\theta_k;\xi_k) + \eta C_L\int_0^1 |\theta_{k+s}-\theta_k||\nabla f(\theta_k; \xi_k)|ds\\
 & = -\eta \nabla F(\theta_k) \cdot \nabla f(\theta_k;\xi_k) +\eta^2 C_L\int_0^1 s|\nabla f(\theta_k; \xi_k)|^2ds\\
 & = -\eta \nabla F(\theta_k) \cdot \nabla f(\theta_k;\xi_k) +\frac{\eta^2}{2} C_L|\nabla f(\theta_k; \xi_k)|^2,
 \end{aligned}
 \end{equation}
and we apply the smoothness assumption (\ref{smooth}) in the middle inequality. Rearrange terms and multiply with the indicate function on both sides, we have
\begin{equation}
    \begin{aligned}        F(\theta_{k+1})\one_{E_{k+1}(r)} \leq \Big(F(\theta_k)-\eta \nabla F(\theta_{k})\cdot \nabla f(\theta_k;\xi_k)+\frac{\eta^2 }{2}C_L|\nabla f(\theta_k; \xi_k)|^2\Big)\one_{E_{k+1}(r)}.
    \end{aligned}
\end{equation}
Now we want to replace $\one_{E_{k+1}(r)}$ by $\one_{E_{k}(r)}$ on the right side.
% \begin{equation}\label{eqn1149}
%     \begin{aligned}
%         f(\theta_{k+1})&\one_{B_{k+1}(r)} \\
%         &\leq \Big(f(\theta_k)-\eta \nabla f(\theta_{k})\cdot g(\theta_k,\xi_k)+\frac{\eta^2 }{2}C_L|g(\theta_k; \xi_k)|^2\Big)\times\Big(\one_{B_{k+1}(r)}- \one_{B_{k}(r)}\Big)\\
%         &+ \Big(f(\theta_k)-\eta \nabla f(\theta_{k})\cdot g(\theta_k,\xi_k)+\frac{\eta^2 }{2}C_L|g(\theta_k; \xi_k)|^2\Big)\one_{B_{k}(r)}.
%     \end{aligned}
% \end{equation}
Note that $\one_{E_{k}(r)}\geq \one_{E_{k+1}(r)}$, and moreover, 
\be 
F(\theta_k)-\eta \nabla F(\theta_{k})\cdot \nabla f(\theta_k;\xi_k)+\frac{\eta^2 }{2}C_L|\nabla f(\theta_k; \xi_k)|^2\geq 0,
\ee
since the discriminant of this quadratic equation is $|\nabla F(\theta_k)|^2-2C_L F(\theta_k) \leq 0$ by (\ref{growth}).
We thus get
\begin{align}\label{jan273}
    F(\theta_{k+1})\one_{E_{k+1}(r)} \leq \Big(F(\theta_k)-\eta \nabla F(\theta_{k})\cdot \nabla f(\theta_k;\xi_k)+\frac{\eta^2 }{2}C_L|\nabla f(\theta_k; \xi_k)|^2\Big)\one_{E_{k}(r)}.
\end{align}
We replace $\nabla f$ above by the decomposition (\ref{assump:jan271}), that is $\nabla f(\theta_k; \xi_k) = \nabla F(\theta_k) +\sqrt{\sigma F(\theta_k)} Z_{\theta, \xi}$, and use the bound (\ref{growth}) as well as the local \L{}ojasiewicz condition (\ref{assump:local_PL}). By taking the expectation we have that
\be 
\E[F(\theta_{k+1})\one_{E_{k+1}(r)}|\mathcal{F}_{k}] \leq \big(1-\eta \alpha + \frac{\eta^2}{2}C_L(2C_L+\sigma)\big) F(\theta_k) \one_{E_{k}(r)}.
\ee
We may  choose a small step-size $0<\eta^* \leq \min\{ \frac{1}{\alpha}, \frac{\alpha}{4C_L(2C_L+\sigma)}\}$, so that 
$$\rho=1-\eta^* \alpha + \frac{\eta^{*2}}{2}C_L(2C_L+\sigma)\leq 1-\frac{\alpha \eta^*}{2}\in(0,1).$$
With this, we obtain a contraction 
\be\label{feb101}
\E[F(\theta_{k+1})\one_{E_{k+1}(r)} |\theta_0] \leq \rho\E[F(\theta_{k})\one_{E_{k}(r)}|\theta_0] \leq \rho^{k+1} F(\theta_0).
\ee
Based on the contraction of $F$, we can nicely control the Euclidean distance between $\theta_k$ and $\theta_0$. Note that 
\begin{equation}\label{feb102}
    \begin{aligned}
        \E[|(\theta_{k+1}-\theta_k)\one_{E_k(r)}|] &=\eta^* \E[|\nabla f(\theta_k;\xi_k)\one_{E_k(r)}|] =\eta^* \E[|(\nabla F(\theta_k) + \sqrt{\sigma F(\theta_k)} Z_{\theta_k, \xi_k})\one_{E_k(r)}|]\\
        &\leq \eta^* \sqrt{2C_L}\sqrt{\E[F(\theta_k)\one_{E_k(r)}]} + \eta^* \sqrt{\E[\sigma F(\theta_k)\one_{E_k(r)}]}\sqrt{\E[|Z_{\theta_k, \xi_k})|^2]}\\
        & = (\sqrt{2C_L}+\sqrt{\sigma})\eta^* \rho^{k/2} F(\theta_0),
    \end{aligned}
\end{equation}
where we apply the Cauchy-Schwarz inequality, the growth bound (\ref{growth}), and (\ref{assump:jan271}-\ref{assump:jan272}). The total distance is thus finite
\be
\E \Bigl[\sum_{k=0}^{\infty}|(\theta_{k+1}-\theta_k)\one_{E_k(r)}| \Bigr] \leq \frac{(\sqrt{2C_L}+\sqrt{\sigma})\eta^*}{1-\sqrt{\rho}} F(\theta_0)<\infty.
\ee
By the Markov's inequality, for any $\tilde\delta\in(0,1)$, we have the length of trajectory bounded by
\be 
\sum_{k=0}^{\infty}|(\theta_{k+1}-\theta_k)\one_{E_k(r)}|\leq \frac{(\sqrt{2C_L}+\sqrt{\sigma})\eta^*}{(1-\sqrt{\rho})\tilde \delta} F(\theta_0),
\ee
with probability at least $1-\tilde \delta$.
This means that as long as the local \L{}ojasiewicz condition holds, the stochastic iterates will converge to a point with high probability. What remains to show that there exists a positive probability for 
\be
E_{\infty}(r) := \bigcap_{j=0}^\infty\{|\theta_j-\theta_0|\leq r\}
\ee
to occur. In fact, we can set $r=R-1$ and argue that there exists a positive probability that all stochastic iterates are trapped in the ball $B(\theta_0,R-1)$. Let us consider the following two cases:
\begin{itemize}
    \item Utilizing (\ref{feb102}),
the chance for the next iterate to escape the ball $B(\theta_0, R)$ is
\begin{equation}\label{jan311}
    \begin{aligned}
        \Pr\big(E_k(R-1)~&\text{occurs but}~\theta_{k+1}\in B^c(\theta_0, R)\big) \leq \Pr(|\theta_{k+1}-\theta_k|\one_{E_k(R-1)}\geq 1)\\
        &\leq \E[|\theta_{k+1}-\theta_k|\one_{E_k(R-1)}]\leq (\sqrt{2C_L}+\sqrt{\sigma})\eta^* \rho^{k/2} F(\theta_0).
    \end{aligned}
\end{equation}
\item Utilizing (\ref{assump:levelset}), the chance for the next iterate to enter the annulus $B(\theta_0,R)\setminus B(\theta_0,R-1)$ is 
\begin{equation}\label{jan312}
    \begin{aligned}
        \Pr\big(E_k(R-1)~&\text{occurs but}~\theta_{k+1}\in B(\theta_0,R)\setminus B(\theta_0,R-1)\big)\\
        &\leq\Pr(F(\theta_{k+1})\one_{E_{k+1}(R)}\geq M_0)\leq \frac{\E[F(\theta_{k+1})\one_{E_{k+1}(R)}]}{M_0}\leq \frac{\rho^{k+1} F(\theta_0)}{M_0}.
    \end{aligned}
\end{equation} 
\end{itemize}
We are ready to conclude, note that the event defined in (\ref{event}) over radius $r=R-1$ is monotonically decreasing,
\be 
E_{k+1}(R-1) \subseteq E_k(R-1).
\ee
Let us define 
\be 
\tilde E_{k+1}(R-1):= E_{k}(R-1) \setminus E_{k+1}(R-1),
\ee
whose probability is indeed the summation of (\ref{jan311}) and (\ref{jan312})
\be 
\Pr(\tilde E_{k+1}(R-1)) \leq F(\theta_0)\Big((\sqrt{2C_L}+\sqrt{\sigma})\eta^* \rho^{k/2}+ \frac{\rho^{k+1} }{M_0}\Big). 
\ee
The probability of the complementary event can be bounded,
\be 
\Pr(E_k^c(R-1)) = \sum_{i=1}^{k} \Pr(\tilde E_i(R-1)) < F(\theta_0)\Big((\sqrt{2C_L}+\sqrt{\sigma}) \frac{\eta^*\sqrt{\rho}}{1-\sqrt{\rho}}+ \frac{\rho^2 }{M_0(1-\rho)}\Big)
\ee
for any $k\in\mathbb{N}$. Thus for every $0<\delta<1$, there exists $\eps>0$ depending on $C_L, \sigma, \eta^*, \rho$ and $M_0$, such that if $F(\theta_0)\leq \eps$, then
\be
\Pr(E_k(R-1))\geq 1-\delta
\ee
for all $k\in\mathbb{N}$. Taking $k\to\infty$, we then have $\Pr(E_{\infty}(R-1))\geq 1-\delta$.

Conditioned on the event $E_{\infty}(R-1)$, we can conclude from (\ref{feb101}) that for every $\beta\in[1,\rho^{-1})$, 
\be 
\lim_{k\to \infty} \beta^k F(\theta_k) =0
\ee
almost surely. Moreover, for all $1\leq j<k$, by (\ref{feb102}),
\begin{equation}
    \begin{aligned}
        \E[|\theta_{k}-\theta_j | \one_{E_{\infty}(R-1)}] &\leq \sum_{i=j}^{k-1}\E[|\theta_{i+1}-\theta_i | \one_{E_{\infty}(R-1)}] \leq \sum_{i=j}^{k-1}\E[|\theta_{i+1}-\theta_i | \one_{E_i(R-1)}] \\
        &\leq \sum_{i=j}^{k-1}(\sqrt{2C_L}+\sqrt{\sigma})\eta^* \rho^{i/2} F(\theta_0)<(\sqrt{2C_L}+\sqrt{\sigma})\eta^*  F(\theta_0)\frac{\rho^{j/2}}{1-\sqrt{\rho}}.
    \end{aligned}
\end{equation}
As $j\to \infty$, this bound goes to zero. Thus $\{\theta_k\}_{k\geq 0}$ from a Cauchy sequence conditioned on $E_{\infty}(R-1)$, and 
$$\theta_k\to \theta_* \in B(\theta_0,R-1),\quad \text{as}~k\to \infty.$$
  By (\ref{feb101}), we see that $F(\theta_*) = 0$ conditioned on $E_{\infty}(R-1)$.  We also know that with probability at least $1-\delta$, $\theta_k\in B(\theta_0,R-1)$ for all $k\in\mathbb{N}$. Therefore, we can conclude that with probability at least $1-\delta$, $\theta_k$ converges to the minimizer $\theta_*\in B(\theta_0,R-1)$.
\end{proof}

\section{Non-convergence with bounded noises}
One may wonder if the machine learning noise is relaxed by only assuming the boundedness (\ref{assump:bdd}), whether it can be shown that SGD iterates still converge inside $B(\theta_0, r)$ for some radius $r>0$. As a companion of bounded noises, we should consider to take a Robbins-Monro type step sizes
\be\label{adpstep}
\eta_k= \frac{\gamma}{(k+n_0)^q},\quad ~q\in(1/2,1]
\ee
for SGD (\ref{alg:sgd}), with constants $\gamma, n_0>0$ to be chosen.

If all the stochastic iterates $\theta_k$ stay inside the ball where the local \L{}ojasiewicz condition (\ref{assump:local_PL}) holds, then we show in Lemma \ref{lem:algebraic_conv} that the convergence happens with an algebraic rate. Lemma \ref{lem:algebraic_conv} relies on classical results in numerical sequences, which can be traced back to \cite{chung1954stochastic}.
\begin{lem}[\cite{chung1954stochastic}, Lemma 1 and Lemma 4]\label{lem:aux}
Let $\{b_k\}_{k\geq1}$ be a non-negative sequence such that
\be\label{ineq:feb13}
b_{k+1}\leq \Big(1-\frac{C_1}{(k+n_0)^q}\Big)b_k+\frac{C_2}{(k+n_0)^{q+p}},
\ee
where $q\in(0,1]$, $p>0$ and $C_1, C_2>0$, then
\begin{enumerate}
    \item if $q=1$ and $C_1>p$, we have
    \be 
    b_k \leq \frac{C_2}{C_1-p}\frac{1}{k} + o\Big(\frac{1}{k}\Big);
    \ee
    \item if $q<1$, we have
    \be 
    b_k \leq \frac{C_2}{C_1}\frac{1}{k^p} + o\Big(\frac{1}{k^p}\Big).
    \ee
\end{enumerate}
\end{lem}

In the proof of Lemma \ref{lem:algebraic_conv}, an iteration inequality of the form (\ref{ineq:feb13}) will show up after we apply the local \L{}ojasiewicz condition (\ref{assump:local_PL}), set up suitable $n_0$, and use the boundedness of the noise (\ref{assump:bdd}). The power $q$ comes from (\ref{adpstep}), and the extra power $p$ comes from $\eta_k^2$.

\begin{lem}\label{lem:algebraic_conv} Given a radius $r>0$ where (\ref{assump:local_PL}) holds.
Assume that the noise term in (\ref{jan271}) satisfies (\ref{assump:bdd}), and the SGD iterates are updated with step sizes $\eta_k = \frac{\gamma}{(k+n_0)^q}, 1/2<q\leq 1$. Conditioned on the event that $\theta_k\in B(\theta_0, r)$ for all $k\in \mathbb{N}$, then for $n_0\geq \Big(\frac{2C_L^2\gamma}{\alpha}\Big)^{1/q}$, we have the convergence
\begin{equation}\label{feb201}
    \begin{aligned}
         \E[F(\theta_{k})] \leq \begin{cases}
    \frac{\gamma^2 c^2 C_L}{\alpha \gamma-2}\frac{1}{k} + o\Big(\frac{1}{k}\Big), &\text{ if $q=1$ and $\gamma>2/\alpha$}\\
    \frac{\gamma^2c^2C_L}{\alpha \gamma}\frac{1}{k^q} + o\Big(\frac{1}{k^q}\Big), &\text{ if $1/2<q<1$}.
    \end{cases} 
    \end{aligned}
\end{equation}
\end{lem}
\begin{proof}
We consider the same event as before
\begin{align}
    E_k(r) := \bigcap_{j=0}^k\{|\theta_j-\theta_0|\leq r\},
\end{align}
for some $0<r\leq R$. Taking similar beginning steps as in Theorem \ref{thm:main}, we arrive to the same inequality as (\ref{jan273}).
\begin{align}
    F(\theta_{k+1})\one_{E_{k+1}(r)} \leq \Big(F(\theta_k)-\eta_k \nabla F(\theta_{k})\cdot \nabla f(\theta_k; \xi_k)+\frac{\eta_k^2 }{2}C_L|\nabla f(\theta_k; \xi_k)|^2\Big)\one_{E_{k}(r)}.
\end{align}
By inserting $\nabla f(\theta; \xi) = \nabla F(\theta)+Z(\theta; \xi)$, we expand the right side to be
\begin{equation}\label{nov121}
 \begin{aligned}
     & F(\theta_{k+1})\one_{E_{k+1}(r)} \\
      &\leq \Big(F(\theta_k)-\eta_k |\nabla F(\theta_{k})|^2+(\eta_k^2C_L-\eta_k )\nabla F(\theta_{k}) \cdot Z(\theta_k; \xi_k)+\frac{\eta_k^2 }{2}C_L(|\nabla F(\theta_k; \xi_k)|^2+|Z(\theta_k; \xi_k)|^2)\Big)\one_{E_{k}(r)}
 \end{aligned}   
\end{equation} 
% Let us relax the assumption of $Z(\theta_k,\xi_k)$ having zero mean, and only utilize the property that its second moment is bounded by $c^2$ (\ref{assump:bdd}). The second term 
Because $Z(\theta_k,\xi_k)$ has zero mean and its second moment is bounded by $c^2$ (\ref{assump:bdd}), 
in addition with (\ref{growth}), by taking the expectation we have that
\be\label{apr7}
\E[F(\theta_{k+1})\one_{E_{k+1}(r)} |\theta_0] \leq \Big(1-\alpha \eta_k+\eta_k^2 C_L^2\Big)\E[F(\theta_{k})\one_{E_{k}(r)}|\theta_0]+\frac{\eta_k^2c^2 C_L}{2}.
\ee
We may view $\E[F(\theta_{k})\one_{E_{k}(r)}|\theta_0]$ as $b_k$ in (\ref{ineq:feb13}). We may set $n_0$ to satisfy
\be 
n_0\geq \Big(\frac{2C_L^2\gamma}{\alpha}\Big)^{1/q},
\ee
so that
\be 
1-\alpha \eta_k+\eta_k^2 C_L^2 \leq 1-\frac{\alpha\eta_k}{2} = 1-\frac{\alpha\gamma}{2(k+n_0)^q}.
\ee
Note that the higher order term in (\ref{apr7} has the expression
\be 
\frac{\eta_k^2c^2 C_L}{2} = \frac{\gamma^2c^2 C_L}{2(k+n_0)^{2q}}.
\ee 
 Thus, we apply Lemma \ref{lem:aux} and obtain the convergence
 \begin{enumerate}
     \item if $q=1$, we set $\gamma>\frac{2}{\alpha}$ so that
     \be 
  \E[F(\theta_{k})\one_{E_{k}(r)}|\theta_0]\leq \frac{\gamma^2 c^2 C_L}{\alpha \gamma-2}\frac{1}{k} + o\Big(\frac{1}{k}\Big);
    \ee
    \item if $1/2<q<1$, we have
    \be 
    \E[F(\theta_{k})\one_{E_{k}(r)}|\theta_0]\leq \frac{\gamma^2c^2C_L}{\alpha \gamma}\frac{1}{k^q} + o\Big(\frac{1}{k^q}\Big).
    \ee
 \end{enumerate}

\end{proof}

\begin{rem}
Compared with the proof of Theorem \ref{thm:main}, the decay rate (\ref{feb201}) of $F(\theta_k)$ is not enough to ensure that $E_{\infty}(r)$ happens with a positive probability. In fact, we need both probabilities in (\ref{jan311}) and (\ref{jan312}) to be summable. 

In \cite{mertikopoulos2020almost}, the authors assume a local convexity in a convex neighborhood $\mathcal{K}$ of a local minimizer $x^*$, so that there exists $0<\alpha<\beta<\infty$ and
\begin{align*}
    \alpha |x-x^*|^2\leq \nabla F(x)^{\top}(x-x^*)\leq \beta |x-x^*|^2,\quad \text{for all}~x\in\mathcal{K}.
\end{align*}
This ensures a strong contraction of $|x_k-x^*|^2$  through iterations. Just with bounded noises, they can show $\{x_k\}$ stays inside the neighborhood $\mathcal{K}$ with a positive probability under Ronbin-Monro step sizes. Unfortunately, the local \L{}ojasiewicz condition (\ref{assump:local_PL}) we assume here is not sufficient to guarantee such a contraction.
\end{rem}

In fact, we can show that the boundedness of noises is not enough to guarantee that iterates stay inside a ball all the time under the local \L{}ojasiewicz condition (\ref{assump:local_PL}). Let us present a negative result in the following. By constructing additive noise $\mathcal{O}_k Z(\theta_k; \xi)\equiv \mathcal{O}_k Z_k$ such that $\E[Z_k] =0$ and $\E[|Z_k|] = \bar{m}>0$ for $k\in\mathbb{N}$, with orthogonal matrices $\mathcal{O}_k$ so that all $\mathcal{O}_kZ_k$ have the same direction, we find that  that under step sizes (\ref{adpstep}), iterates will escape the ball $B(\theta_0,r)$ almost surely.

\begin{thm}\label{lem:counterexample}
Consider the following stochastic gradient iteration
\begin{align}\label{alg:sgd2}
    \theta_{k+1} = \theta_k - \eta_k (\nabla F(\theta_k)+\mathcal{O}_kZ_k),
\end{align}
where $\eta_k=\frac{\gamma}{(k+n_0)^q}, 1/2<q\leq 1$, $n_0\geq \Big(\frac{2C_L^2\gamma}{\alpha}\Big)^{1/q}$ and $\gamma>2/\alpha$ as in Lemma \ref{lem:algebraic_conv}. We assume that $Z_k\in\R^d$ are i.i.d. random vectors with $\E[Z_k] = 0, \E[|Z_k|]=\bar m$ for some $\bar m>0$, and $\E[|Z_k|^2]\leq c^2$ for some $c>0$. Furthermore,  $\mathcal{O}_k\in\R^{d\times d}$ are orthogonal matrices which rotate $Z_k$ to the direction of $Z_1$, and we set $O_0=I$. Then, given a radius $r>0$ such that (\ref{assump:local_PL}) holds,  $\{\theta_k\}_{k\geq 0}$ will exit the ball $B(\theta_0, r)$ almost surely.
\end{thm}
\begin{proof}
    We argue by contradiction. Suppose all iterates from (\ref{alg:sgd2}) stay inside the ball $B(\theta_0,r)$, we aim to prove that $\big|\lim_{n\to\infty}\sum_{k=0}^n(\theta_{k+1}-\theta_k)\big|=\big|\lim_{n\to\infty}\theta_{n+1}-\theta_0\big|=\infty$ almost surely. For any $n\geq 1$, we have
\begin{equation}\label{feb211}
    \begin{aligned}
        \big|\sum_{k=0}^n(\theta_{k+1}-\theta_k-\eta_k\nabla F(\theta_k))\big| &= \big|\sum_{k=0}^n \eta_k \mathcal{O}_kZ_k\big| = \sum_{k=0}^n \eta_k |Z_k|\\
        &\geq \sum_{k=0}^n \frac{\gamma}{k+n_0}|Z_k|\geq \tilde c \sum_{k=1}^n \frac{1}{k}|Z_k|
    \end{aligned}
\end{equation}
where the second inequality can be satisfied if we choose $\tilde c\in(0, \frac{\gamma}{n_0+1})$. We claim that 
\begin{align}\label{feb212}
    \sum_{k=1}^n \frac{1}{k}|Z_k| \to \infty,\quad \text{as}~n\to \infty ~~\text{almost surely}.
\end{align}
If not, say there exists a finite number $\mu>0$ such that $\lim_{n\to \infty}\sum_{k=1}^n \frac{1}{k}|Z_k| =\mu$, the Ces\`{a}ro mean theorem implies that the Ces\`{a}ro average also converges to the same limit,
\begin{align}
    \frac{1}{n}\sum_{j=1}^n \Big(\sum_{k=1}^j \frac{1}{k}|Z_k| \Big) \to \mu,\quad \text{as}~n\to \infty ~~\text{almost surely}.
\end{align}
On the other hand,
\begin{equation}
\begin{aligned}
  \frac{1}{n}\sum_{j=1}^n \Big(\sum_{k=1}^j \frac{1}{k}|Z_k| \Big) &= \frac{1}{n}\sum_{k=1}^n \Big(\sum_{j=k}^n \frac{1}{k}|Z_k| \Big) = \frac{1}{n}\sum_{k=1}^n\frac{n+1-k}{k} |Z_k|\\
  &=\sum_{k=1}^n\frac{1}{k}|Z_k| +\frac{1}{n}\sum_{k=1}^n\frac{1}{k}|Z_k| - \frac{1}{n}\sum_{k=1}^n|Z_k|,
\end{aligned}
\end{equation}
which implies that as $n\to \infty$, $\frac{1}{n}\sum_{k=1}^n|Z_k|\to 0$ almost surely. But by the law of large number and the construction of $Z_k$, $\frac{1}{n}\sum_{k=1}^n|Z_k|\to \bar m \neq 0$, we get the contradiction. Thus (\ref{feb212}) is proved.

Back to (\ref{feb211}), we note that
\begin{equation}\label{nov122}
    \big|\sum_{k=0}^n(\theta_{k+1}-\theta_k-\eta_k\nabla F(\theta_k))\big|\leq \big|\sum_{k=0}^n(\theta_{k+1}-\theta_k)\big|+\sum_{k=0}^n\eta_k|\nabla F(\theta_k)|
\end{equation}
Due to the convergence result (\ref{feb201}) in Lemma \ref{lem:algebraic_conv}, we have 
\begin{equation}\label{nov123}
    \lim_{n\to\infty}\sum_{k=1}^n \eta_k|\nabla F(\theta_k)|\leq \lim_{n\to\infty}\sum_{k=1}^n \frac{\gamma\sqrt{2C_L}}{(k+n_0)^q}\sqrt{F(\theta_k)}<\infty
\end{equation} almost surely. Combining (\ref{feb212}), (\ref{nov122}) and (\ref{nov123}) together, we can deduce from (\ref{feb211}) that $\lim_{n\to\infty}\big|\sum_{k=0}^n(\theta_{k+1}-\theta_k)\big|=\big|\lim_{n\to\infty}\theta_{n+1}-\theta_0\big|=\infty$ almost surely. which means that the iterates of the SGD algorithm (\ref{alg:sgd2}) will exit the ball almost surely. 

% However, this statement cannot happen if all iterates stay inside the ball, because by Lemma \ref{lem:algebraic_conv}, $F(\theta_k)$ converges. By a stopping criterion, it implies that $\sum_{k=0}^n|\theta_{k+1}-\theta_k|<\infty$. The infinite sum tells that the SGD iterates will exit the ball almost surely.
    % Lemma \ref{lem:algebraic_conv} ensures that $F(\theta_k)$ converges algebraically if all iterates stay inside $B(\theta_0, r)$, due to the stopping criterion $F(\theta_k)<\eps$, we deduce that $\sum_{k=0}^{\infty}|\theta_{k+1}-\theta_k|<\infty$;
\end{proof}

\subsection*{Acknowledgement}
The work of JL is supported in part by 
National Science Foundation via grant DMS-2012286.

\appendix
\section{Assumption 3 verification}
In \cite{chatterjee2022convergence}, Chatterjee provides a feedforward neural network that satisfies the local \L{}ojasiewicz condition (\ref{assump:local_PL}). The purpose of this appendix is to show that, by restricting the training parameter $\theta$ in some subspace, this feedforeard neural network construction satisfies Assumption 3 (\ref{assump:levelset}) as well.

The loss function that \cite{chatterjee2022convergence} considers is the squared error loss. For the dateset $\{(x_i,y_i)\}_{i=1}^n$, we consider
\begin{equation}\label{mar131}
    F(\theta) := \frac{1}{n}\sum_{i=1}^n (y_i - \phi(x_i,\theta))^2.
\end{equation}
Here, $\phi(x,\theta)$ is the neural network writes as
\begin{equation}\label{mar132}
    \phi(x,\theta):= \tilde{\sigma}_L(W_{L}\tilde{\sigma}_{L-1}(\cdots (W_2\tilde{\sigma}_1(W_1 x+b_1)+b_2)\cdots)+b_L),
\end{equation}
with $\tilde{\sigma}_1\cdots, \tilde{\sigma}_L$ being activation functions acting componentwisely, and $W_l\in \R^{d_l\times d_{l-1}}, d_0=d, d_L=1, b_l\in\R^{d_l}$ for $1\leq l \leq L$. The parameter to be trained is 
\begin{equation}
    \theta = (W_1, b_1, W_2, b_2,\cdots, W_L, b_{L}),
\end{equation}
and it can be viewed as a vector in $\R^p$ with $p = \sum_{l=1}^L d_l(d_{l-1}+1)$.

We assume that the input data $x_1,\cdots, x_n \in \R^d$ are linearly independent, and consider $\frac{1}{n}X^{\top} X$, where the matrix $X=(x_1,\cdots x_n)\in\R^{d\times n}$ is formed by column vectors $x_i$'s. When $d\geq n$, the matrix $\frac{1}{n}X^{\top} X$ has a positive minimum eigenvalue denoted by $\lambda_0>0$. 

Theorem 2.1 in \cite{chatterjee2022convergence} shows that, with an appropriate initialization and neural network setups, the landscape of $F(\theta)$ satisfies the local \L{}ojasiewicz condition (\ref{assump:local_PL}), and the gradient descent will converge to $F(\theta^*)=0$ by the convergence analysis. However, the setup in \cite{chatterjee2022convergence} does not exclude the case of (globally) flat minima, where Assumption 3 (\ref{assump:levelset}) may fail. In order to allow assumptions (\ref{assump:local_PL}) and (\ref{assump:levelset}) to coexist, we restrict the domain of training $\theta$ to some subspace containing $\theta_0$ and $\theta^*$ where Assumption 3 (\ref{assump:levelset}) holds.

One choice for such subspace is the following: fix some $\tilde{\alpha}>0$. 
\begin{equation}\label{Theta_subspace}
    \Theta = \{\theta'\in \R^p: (W_1' x_i+b_1')_j\geq \tilde\alpha |\theta'-\theta_0|, \;\text{for}~1\leq i\leq n, ~1\leq j\leq d_1, \text{and}~ b_l'\geq 0, ~2\leq l\leq L \}.
\end{equation}
Such a $\tilde{\alpha}>0$ can exist if we have a suitable data matrix $X$.
% In order to ensure all iterates $\theta_k$ lie in $\Theta$, we can apply the acceptance and rejection strategy during the SGD training. 
The setup of $\Theta$ is to avoid training $\theta'$  in the nullspace of $\phi(x,\cdot)$. $\Theta$ looks somewhat artificial in order to serve  the computational convenience (see Theorem \ref{thm:appendix}). In practice during the training for deep neural networks, all gradient descent iterates can be considered to stay in $\Theta$, as otherwise the the gradient becomes small and  training would terminate.
 We leave the following remark as one  example of such suitable neural networks.
\begin{rem}
The feedforward neural networks in \cite{chatterjee2022convergence} include deep linear neural networks. Taking biases to be zeroes and activation functions to be identity, it writes as
\begin{equation}
    \tilde{\phi}(\theta,x):= W_L W_{L-1}\cdots W_1 x.
\end{equation}
The partial derivative of $F(\theta)$ with respect to $W_l$, for each $2\leq l\leq L$ is given as
\begin{equation}
    \frac{\d F}{d W_l} = \frac{2}{n}\sum_{i=1}^n W_{l+1}^\top \cdots W_{L}^\top \big(W_L W_{L-1}\cdots W_1 x_i-y_i\big) x_i^\top W_1^\top\cdots W_{l-1}^\top,
\end{equation}
and when $l=L$, we take $W_{L+1}^\top W_L^\top$ as an identity matrix by convention (so the above formula applies to every $l$).
Then with the initialization in Theorem \ref{thm:appendix}, we see that $\frac{\d F}{d W_l}$ will be small if $W_1 x_i$ is small of all $1\leq i\leq n$. When it happens, we may stop training due to the vanishing $\nabla F$. 
\end{rem}

Let us recall the setups in Theorem 2.1 in \cite{chatterjee2022convergence} and show that the loss function $F(\theta)$ satisfies Assumption 3 (\ref{assump:levelset}) restricted in $\Theta$.
\begin{thm}\label{thm:appendix}
    We consider the squared error loss (\ref{mar131}) with a feedforward neural network (\ref{mar132}), with depth $L\geq 2$, $d_L=1$ and $\tilde\sigma_L=$ identity. Suppose for $0\leq l\leq L-1$, the activation functions $\tilde \sigma_l \in C^2, \tilde \sigma_l(0)=0$, and $c_l:=\min_{x}\tilde \sigma'_l(x)>0$. Suppose the input data $x_1,\cdots, x_n \in \R^d$ are linearly independent, and let $\lambda_0$ be the minimum eigenvalue of $\frac{1}{n}X^{\top} X$, with the matrix $X=(x_1,\cdots x_n)\in\R^{d\times n}$ formed by column vectors $x_i$'s. We can initialize $\theta_0 = (W_1, b_1, W_2, b_2,\cdots W_L, b_{L})$ such that $b_l=0$ for all $1\leq l\leq L$, $W_1=0$, and the entries of $W_2,\cdots, W_{L-1}$ are all strictly positive. In addition, let $R>0$ be the minimum value of the entries of $W_2,\cdots, W_{L-1}$, and $A>R/2>0$ be the minimum value of the entries of $W_L$. Then for any $\theta'\in B(\theta_0, R/2)\cap\Theta$, we have the following lower bound:
\begin{equation}
F(\theta')\geq \frac{\tilde{\alpha}^2}{2} (A-R/2)^2 (R/2)^{2L-4}(c_{L-1}\cdots c_2 c_1 d_{L-1}\cdots d_1 )^2 |\theta'-\theta_0|^2 - \frac{1}{n}\sum_{i=1}^n y_i^2.
\end{equation}
\end{thm}

\begin{proof}
We define
\begin{equation}
    \phi_1 (x,\theta) = \tilde{\sigma}_1(W_1x+b_1)
\end{equation}
and for $2\leq l\leq L$,
\begin{equation}
    \phi_l (x,\theta) = \tilde{\sigma}_l(W_{l}\tilde{\sigma}_{l-1}(\cdots (W_2\tilde{\sigma}_1(W_1 x+b_1)+b_2)\cdots)+b_l),
\end{equation}
    so that $\phi_L = \phi$. 

Let $\theta_0$ be the starting vector where the entries of $W_2,\cdots W_{L-1}$ are all strictly positive, and $b_1,\cdots, b_L$ and $W_1$ be zero, then one can find $\phi(x,\theta_0)=0$ for each $x$. Thus $F(\theta_0) = \frac{1}{n}\sum_{i=1}^n y_i^2$.
    Using the Young's inequality, the loss function can be rewritten as
    \begin{equation}
    \begin{aligned}
        F(\theta')&\geq \frac{1}{n}\sum_{i=1}^n \big(y_i^2 +\phi(x_i,\theta')^2-2|\phi(x_i,\theta')||y_i|\big)\geq \frac{1}{2n}\sum_{i=1}^n \phi(x_i,\theta')^2 - \frac{1}{n}\sum_{i=1}^n y_i^2 
        \end{aligned}
    \end{equation}
For each $x_i$ and $1\leq j \leq d_1$, since $\theta'\in\Theta$, we have that
\begin{equation}
    \begin{aligned}
        \phi_1(x_i,\theta')_j \geq c_1 \tilde{\alpha} |\theta'-\theta_0|.
    \end{aligned}
\end{equation}
Given the choice of $\theta_0$, we note that all the entries of $W_2, \cdots, W_{L-1}$ are bounded below by a constant $R>0$, and $b_1,\cdots, b_L$ and $W_1$ are zero. If we take any $\theta'\in \Theta$ such that $|\theta'-\theta_0|\leq R/2$, then for such $ \theta'=(W_1', b_1',  W_2', b_2',\cdots W_L',  b_{L}')\in \Theta$, all the entries of $ W_2', \cdots,  W_{L-1}'$ are bounded below by $R/2$, and $b_2',\cdots, b_L'\geq 0$. In the second layer, for $1\leq j \leq d_2$, we have that
\begin{equation}
    \phi_2(x_i,\theta')_j \geq c_2 c_1 d_1\frac{R}{2} \tilde{\alpha} |\theta'-\theta_0|.
\end{equation}
Iterate it up to $(L-1)$th layer, we get that for $1\leq j \leq d_{L-1}$,
\begin{equation}
    \phi_{L-1}(x_i,\theta')_j \geq c_{L-1}\cdots c_2 c_1 d_{L-2}\cdots d_1(R/2)^{L-2} \tilde{\alpha} |\theta'-\theta_0|.
\end{equation}
In the last layer, as $A>R/2$ is a lower bound on the entries of $W_L$,  the entries of $ W_L'$ are bounded below by $A-R/2$. Therefore,
\begin{equation}
   \phi(x_i,\theta')= \phi_L(x_i,\theta')\geq (A-R/2)(R/2)^{L-2}c_{L-1}\cdots c_2 c_1 d_{L-1}\cdots d_1 \tilde{\alpha} |\theta'-\theta_0|,
\end{equation}
and we can bound $F(\theta')$ by
\begin{equation}
    F(\theta')\geq \frac{\tilde{\alpha}^2}{2} (A-R/2)^2 (R/2)^{2L-4}(c_{L-1}\cdots c_2 c_1 d_{L-1}\cdots d_1 )^2 |\theta'-\theta_0|^2 - \frac{1}{n}\sum_{i=1}^n y_i^2.
\end{equation}
\end{proof}
From the lower bound estimate above, if $R$ is large enough compared with $\frac{1}{n}\sum_{i=1}^n y_i^2$, we can find $M_0>0$ such that $F(\theta')\geq M_0$ for $R/2-1\leq |\theta'-\theta_0|\leq R/2$.

\section{Proof of Lemma \ref{auxiliary0}}
\begin{proof}
    The statement is trivially true of $\nabla F(\theta)=0$, so we only consider the case where $\nabla F(\theta)\neq 0$. Let us define a function 
    \begin{equation}
        g(t) = F(\theta - t\nu),\quad \nu = \frac{\nabla F(\theta)}{|\nabla F(\theta)|},
    \end{equation}
    then $g'(0) = -\nu\cdot \nabla F(\theta) = |\nabla F(\theta)|$ and 
    \begin{equation}
        |g'(t)-g'(0)| \leq |\nabla F(\theta-t\nu)-\nabla F(\theta)|\leq C_L t,
    \end{equation}
    if $\theta, \theta-t\nu \in \mathcal{K}$ based on (\ref{smooth}). With that,
    \begin{equation}
        g(t) = g(0) +\int_0^t g'(s) ds \leq F(\theta)-|\nabla F(\theta)|t+\frac{C_L}{2} t^2.
        \end{equation}
  By taking $t= -\frac{|\nabla F(\theta)|}{C_L}$, the right side bound achieves the minimum and it implies that
  \begin{equation}
      \text{RHS} = F(\theta)-\frac{|\nabla F(\theta)|^2}{2 C_L}\geq 0.
  \end{equation}
  Because we require both $\theta, \theta-t\nu \in \mathcal{K}$, due to the choice of $t$, we may choose a compact set $\tilde{\mathcal{K}}\subset \mathcal{K}$, with $\text{dist}(\d\mathcal{K}, \d\mathcal{\tilde K})\geq \frac{\bar{C}}{C_L}$, so that $\theta, \theta-t\nu \in \mathcal{K}$ for all $\theta\in \tilde{\mathcal{K}}$.
\end{proof}

\bibliographystyle{abbrv} 
\bibliography{references}

\end{document}